\documentclass[a4paper,fleqn,11pt]{cas-sc}
\usepackage[authoryear,longnamesfirst]{natbib}
\usepackage{verbatim} 
\usepackage{amsthm}
\usepackage{times}
\usepackage{amsmath,amsfonts,amssymb}
\usepackage{graphicx}
\usepackage{booktabs}
\usepackage{url}
\usepackage{algorithm}
\usepackage{algpseudocode}
\usepackage{xcolor}
\usepackage{float}
\usepackage{setspace}
\theoremstyle{plain}  
\newtheorem{definition}{Definition}
\newtheorem{proposition}{Proposition}
\title[mode=title]{Structural Segmentation of the Minimum Set Cover Problem:
Exploiting Universe Decomposability for Metaheuristic Optimization}

\begin{document}
\let\WriteBookmarks\relax
\def\floatpagepagefraction{1}
\def\textpagefraction{.001}
\shorttitle{Exploiting Universe Decomposability in MSCP}
\author[1]{Isidora Hern\'andez}

\author[1]{H\'ector Ferrada}[orcid=0000-0002-8334-4540]
\ead{hector.ferrada@uach.cl}
\cormark[1]                
\cortext[cor1]{Corresponding author}  

\author[1]{Crist\'obal A. Navarro}[orcid=0000-0001-7090-9904]

\affiliation[1]{organization={Austral University of Chile},
               city={Valdivia},
               country={Chile}}

\maketitle
\begin{abstract}
The Minimum Set Cover Problem (MSCP) is a classical NP-hard combinatorial optimization problem with numerous applications in science and engineering. Although a wide range of exact, approximate, and metaheuristic approaches have been proposed, most methods implicitly treat MSCP instances as monolithic, overlooking potential intrinsic structural properties of the universe.
In this work, we investigate the concept of \emph{universe segmentability} in the MSCP and analyze how intrinsic structural decomposition (universe segmentability) can be exploited to enhance heuristic optimization. We propose an efficient preprocessing strategy based on disjoint-set union (union--find) to detect connected components induced by element co-occurrence within subsets, enabling the decomposition of the original instance into independent subproblems. Each subproblem is solved using the GRASP metaheuristic, and partial solutions are combined without compromising feasibility.
Extensive experiments on standard benchmark instances and large-scale synthetic datasets show that exploiting natural universe segmentation consistently improves solution quality and scalability, particularly for large and structurally decomposable instances. These gains are supported by a succinct bit-level set representation that enables efficient set operations, making the proposed approach computationally practical at scale.
\end{abstract}

\section{Introduction}

The Minimum Set Cover Problem (MSCP) is a fundamental combinatorial optimization problem with applications in crew scheduling \cite{Caprara1997}, network design \cite{BramelSimchiLevi1997}, routing \cite{BramelSimchiLevi1997}, data mining \cite{LanDePuy2007}, resource allocation \cite{BautistaPereira2007}, and so on. 
Formally, given a universe of elements $\mathcal{X}$ and a family of subsets $\mathcal{F}$ whose union covers $\mathcal{X}$, the objective is to select the minimum number of subsets of $\mathcal{F}$ whose union also covers all the elements of $\mathcal{X}$. 
Figure~\ref{fig:example} illustrates an example.

The MSCP is NP-hard, and even its unicost variant, where all subsets have identical costs, admits no polynomial-time exact algorithm unless $P = NP$ \cite{GareyJohnson1979}.
Due to this computational complexity, a vast body of research has focused on the development of approximation algorithms, heuristics, and metaheuristics capable of producing high-quality solutions within reasonable computational time. Classical greedy algorithms \cite{Chvatal1979}, Lagrangian relaxation and dual-based approaches \cite{BalasHo1980,FisherKedia1990}, and modern metaheuristics such as GRASP \cite{Feo1995grasp,ResendeRibeiro2010}, ant colony optimization \cite{Ren2010}, and electromagnetism-inspired heuristics \cite{NajiAzimi2010} have demonstrated strong empirical performance across a wide range of benchmark instances. Most of these proposals include standard reference sets for experimental evaluation drawn from the OR-Library \cite{Beasley1990}.

\begin{figure}
\centering
\includegraphics[width=0.55\textwidth]{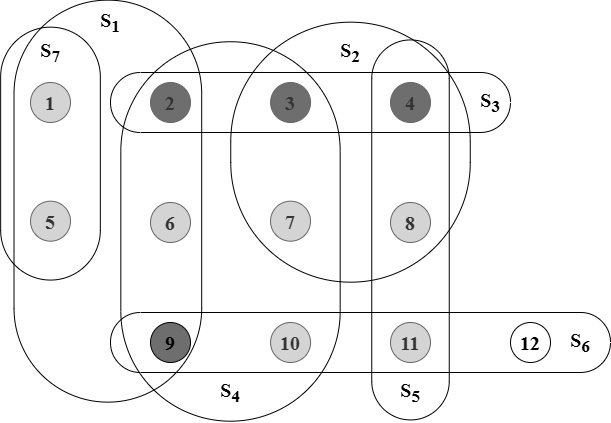}
\caption{Illustration of an instance $(\mathcal{X}, \mathcal{F})$ of the MSCP, with $|\mathcal{X}|=12$ elements and $\mathcal{F}=\{S_1,\ldots,S_7\}$, where $S_1 = \{1, 2, 5, 6, 9\}, S_2 = \{3, 4, 7, 8\}, S_3 = \{2, 3, 4\}, S_4 = \{2, 3, 6, 7, 9, 10\}, S_5 = \{4, 8, 11\}, S_6 = \{9, 10, 11, 12\}, S_7 = \{1, 5\}$. The optimal solution for this example is the minimum set-covering $\mathcal{C^*}=\{S_1,S_2,S_6\}$.}
\label{fig:example}
\end{figure}

Despite these advances, most existing approaches implicitly assume that MSCP instances are structurally indivisible and must be solved as monolithic optimization problems. In practice, however, many instances exhibit latent structural properties that are not explicitly exploited by standard solution methods. In particular, elements of the universe may form groups that never co-occur within any subset, or whose interactions are weakly connected through a limited number of subsets. Treating such instances as indivisible problems leads to unnecessarily large search spaces and increased computational effort, especially for large-scale instances \cite{Yelbay2015}.
Despite decades of research on MSCP solution methods, the question of whether the universe itself can be systematically decomposed into independent components—while preserving feasibility—has received little attention. Existing works primarily focus on improving heuristic operators or search strategies \cite{Caprara1999,Beasley1987}, rather than exploiting intrinsic structural properties of the problem instance. This observation motivates the central research question of this article: \emph{can intrinsic structural properties of MSCP instances be leveraged to decompose the universe into independent subproblems, and does such decomposition improve heuristic solution quality and scalability?}

To address the previos question, we introduce the notion of \emph{universe segmentability} in the MSCP. We model interactions between elements through co-occurrence relationships induced by the subsets and propose an efficient preprocessing strategy based on disjoint-set union (union--find) \cite{GalilItaliano1991} to identify connected components within the universe. When such components exist, the original MSCP instance can be decomposed into independent subinstances that may be solved separately.
We integrate this segmentation strategy with the GRASP metaheuristic, solving each component independently and merging the resulting partial solutions into a global cover. This approach preserves feasibility by construction, reduces the effective problem size, and naturally enables parallel execution. Our implementation builds upon a succinct bit-level set representation previously proposed by Delgado et al.~\cite{DelgadoFerradaNavarro2024}, which supports constant-time set intersection and coverage operations. While not a contribution of this work, this representation plays a crucial role in making the proposed segmentation strategy computationally efficient and scalable.
An alternative segmentation strategic based on artificial or forced partitioning was also explored; however, the tested approach, which involves forcing a partition of the universe into two balanced groups, does not produce improvements in performance,  it is presented in the appendix as negative result for future refinement.

The main contributions of this article are as follows: we formalize the concept of universe segmentability in the MSCP and propose an efficient union-find-based method to detect intrinsic structural decomposition; we demonstrate how natural universe segmentation can be seamlessly integrated into a metaheuristic framework without altering solution feasibility; through extensive experimental evaluation on benchmark instances from the OR-Library \cite{Beasley1990}, we demonstrate that exploiting intrinsic structural properties significantly improves solution quality and scalability for large MSCP instances; and we initiate the search for improved solutions from strong initial approximations while leveraging HPC techniques ---including CPU parallelism, succinct data representations, and bitwise operations--- to accelerate the overall solution process.

The remainder of the article is structured as follows. The state of the art and related work are reviewed in Section \ref{sec:related}. Section \ref{sec:segmentation} describes the properties that we will exploit in Section \ref{sec:grasp}, where we detail our algorithm. Section \ref{sec:results} presents the experiments conducted using both our own datasets and traditional datasets from the literature, with the aim of evaluating the algorithm's performance.
Finally, Section \ref{sec:conclusions} analyzes and discusses the results obtained and presents the conclusions of the study and future work.

\section{Related Work}
\label{sec:related}

The Minimum Set Cover Problem (MSCP) has been extensively studied in the literature due to its theoretical relevance and broad range of practical applications. Existing solution approaches can be broadly categorized into exact algorithms, approximation and greedy-based heuristics, and metaheuristic frameworks. In addition, several preprocessing and reduction techniques have been proposed to improve scalability, although structural decomposition of the universe has received comparatively little attention.

\subsection*{Exact algorithms and mathematical programming approaches}

Exact algorithms for the MSCP aim at guaranteeing optimality, but their applicability is inherently limited by the NP-hard nature of the problem \cite{GareyJohnson1979}. Early and influential work by Beasley \cite{Beasley1987} introduced algorithms combining problem reduction tests, subgradient optimization, and linear programming relaxations, establishing a benchmark for exact MSCP solvers. Related approaches based on cutting planes \cite{BalasHo1980} and branch-and-bound schemes \cite{BalasCarrera1996} further improved computational performance by exploiting dual information~\cite{FisherKedia1990}.
Although these methods guarantee optimal solutions, through exhaustive exploration with pruning heuristics, the exponential growth of the search space severely restricts the scalability of exact algorithms, rendering them impractical for large-scale instances commonly encountered in real-world applications. As a consequence, much of the subsequent research has focused on heuristic and approximate solution methods.

\subsection*{Greedy and approximation-based heuristics}

Among approximation approaches, Chvátal’s greedy algorithm \cite{Chvatal1979} remains one of the most influential methods for the MSCP, due to its simplicity, polynomial-time complexity, and optimal logarithmic approximation guarantee. The greedy paradigm has inspired numerous refinements aimed at improving empirical solution quality while preserving computational efficiency.

Recent developments include greedy heuristics enhanced with probabilistic or information-theoretic mechanisms, such as the surprisal-based greedy strategy proposed by Adamo et al.~\cite{Adamo2023}, which adapts selection criteria based on the information content of uncovered elements. Other improvements focus on data representation and algorithmic efficiency, for example Delgado et al.~\cite{DelgadoFerradaNavarro2024} introduce a greedy-based algorithm that exploits the incidence structure between elements and subsets to guide the selection process. Rather than selecting subsets directly from $\mathcal{F}$, the method prioritizes elements according to their degree, defined as the number of subsets that cover them, and evaluates candidate subsets by progressively covering elements with smaller degrees first.
This strategy aims to prioritize sparsely covered elements, reducing the risk of delaying the coverage of critical elements with few available covering subsets, thereby leading to solutions with improved cardinality. A succinct bit-level representation is employed to support efficient set operations and scalability.

These approaches demonstrate that even within the greedy framework, significant performance gains can be achieved by leveraging structural and computational insights.

\subsection*{Metaheuristics for the MSCP}

To overcome the limitations of deterministic heuristics, a wide variety of metaheuristic algorithms have been proposed for the MSCP. These methods aim to balance exploration and exploitation in large and complex search spaces, often yielding high-quality solutions for large instances. Representative examples include genetic algorithms \cite{Beasley1987}, ant colony optimization \cite{Ren2010}, electromagnetism-inspired metaheuristics \cite{NajiAzimi2010}, and GRASP-based approaches \cite{Feo1995grasp,ResendeRibeiro2010}.

Several GRASP variants have been specifically tailored to the MSCP, including the unicost formulation \cite{BautistaPereira2007}. For instance, the more recent scheme of Reyes and Araya \cite{ReyesAraya2021} incorporates a divide-and-conquer strategy that randomly splits the solution in each iteration into two halves and then makes each one a feasible solution.

While these metaheuristic approaches often produce near-optimal solutions for large MSCP instances, their performance is primarily assessed through empirical evaluation on benchmark datasets, and they generally provide limited theoretical guarantees on solution quality. This is particularly true for GRASP-based and hybrid methods, whose effectiveness strongly depends on instance structure and parameter tuning \cite{Feo1995grasp,ResendeRibeiro2010}.

\subsection*{Preprocessing, reduction, and structural properties}

Beyond solution algorithms, preprocessing and reduction techniques have been explored as a means to simplify MSCP instances prior to optimization. Classical reduction tests are commonly embedded within exact solvers \cite{Beasley1987}, and dual-based information has been shown to significantly impact heuristic performance \cite{Yelbay2015}. However, these techniques primarily aim at eliminating dominated subsets or redundant elements, rather than exploiting deeper structural properties of the universe.

In contrast, the idea of decomposing a combinatorial optimization problem into independent subproblems has been successfully applied in other domains, such as vehicle routing and crew scheduling \cite{Caprara1997}. Nevertheless, for the MSCP, the question of whether the universe itself can be naturally partitioned into independent components—while preserving feasibility—has received little systematic attention. Most existing approaches implicitly treat MSCP instances as structurally indivisible, solving them as monolithic problems.

This gap in the literature motivates the contribution of the present work, which investigates intrinsic structural properties of MSCP instances and their potential to enable principled universe decomposition as a preprocessing step within a metaheuristic framework.

\section{Structural Segmentation of the MSCP}
\label{sec:segmentation}

In this section, we introduce the concept of \emph{structural segmentation} of the Minimum Set Cover Problem (MSCP). The key idea is to exploit intrinsic relationships among elements of the universe induced by subset co-occurrences, with the goal of identifying independent subinstances that can be solved separately without compromising feasibility.

\subsection{Universe Segmentability}

Let $\mathcal{X}$ denote the universe of elements and $\mathcal{F} = \{S_1, \ldots, S_m\}$ the family of subsets whose union covers $\mathcal{X}$. Two elements $x_i, x_j \in \mathcal{X}$ are said to \emph{co-occur} if there exists at least one subset $S_k \in \mathcal{F}$ such that $\{x_i, x_j\} \subseteq S_k$.
We model the co-occurrence relationships among elements by defining an undirected graph $G = (\mathcal{X}, E)$, where each vertex corresponds to an element of the universe, and an edge $(x_i, x_j) \in E$ exists if and only if $x_i$ and $x_j$ co-occur in at least one subset. An illustration of this construction is shown in Fig.\ref{fig:cooccGraph}
\begin{definition}[Universe Segmentability]
An MSCP instance $(\mathcal{X}, \mathcal{F})$ is said to be \emph{segmentable} if the co-occurrence graph $G$ is disconnected. Each connected component of $G$ then defines a subset of elements that is independent of the others.
\end{definition}

\begin{figure}
\begin{center}
\includegraphics[width=0.9\textwidth]{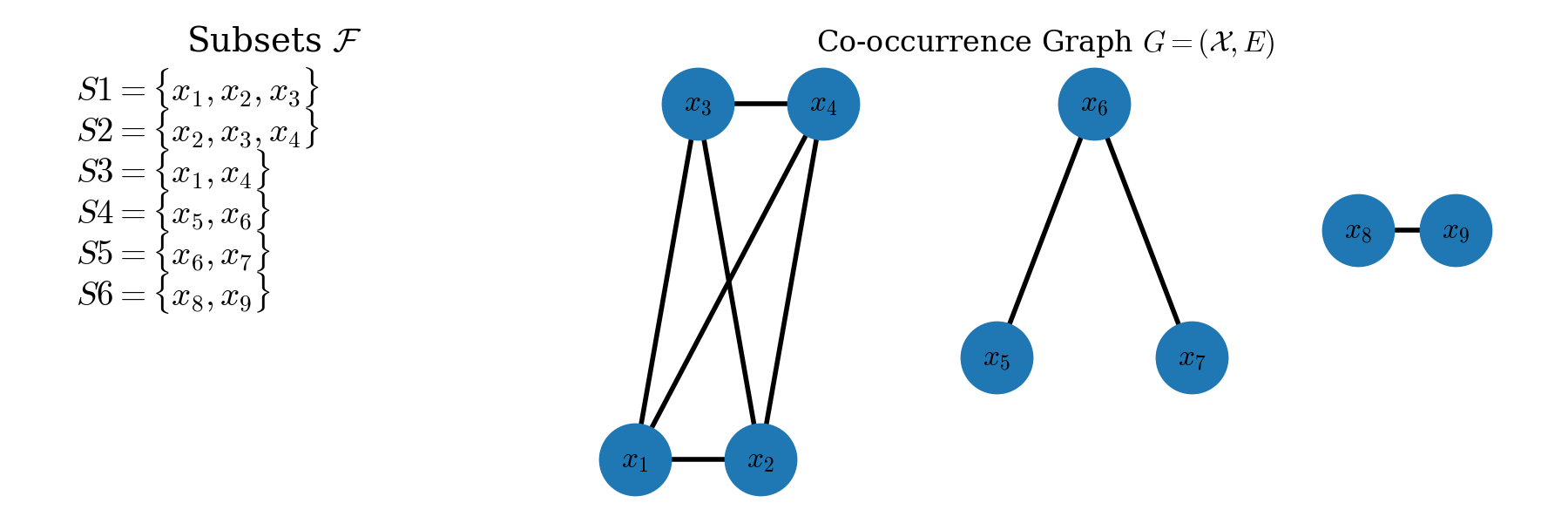}
\end{center}
\vspace*{-5mm}
\caption{Illustrative example of a MSCP instance. (Left) Family of subsets $\mathcal{F}$ defined over the universe $\mathcal{X}$. (Right) Co-occurrence graph induced by $\mathcal{F}$, where an edge connects two  elements if they appear together in at least one subset. The graph decomposes into three connected components, motivating universe segmentation into  independent subinstances.}
\label{fig:cooccGraph}
\end{figure}

Let $\mathcal{X}_1, \ldots, \mathcal{X}_k$ denote the connected components of $G$. For each component $\mathcal{X}_i$, we define a corresponding subfamily $\mathcal{F}_i = \{ S \in \mathcal{F} \mid S \cap \mathcal{X}_i \neq \emptyset \}$. Fig.~2 conceptually illustrates the decomposition of an MSCP instance into independent subinstances.

\subsection{Feasibility Preservation}

The following result establishes that segmentation preserves feasibility.
\begin{proposition}
Let $(\mathcal{X}, \mathcal{F})$ be an MSCP instance whose co-occurrence graph decomposes into connected components $\mathcal{X}_1, \ldots, \mathcal{X}_k$. For each $i$, let $\mathcal{C}_i \subseteq \mathcal{F}_i$ be a feasible cover of $\mathcal{X}_i$. Then $\mathcal{C} = \bigcup_{i=1}^k \mathcal{C}_i$ is a feasible cover of $\mathcal{X}$.
\end{proposition}

\begin{proof}
By construction, every element $x \in \mathcal{X}$ belongs to exactly one connected component $\mathcal{X}_i$. Since $\mathcal{C}_i$ covers all elements of $\mathcal{X}_i$, $x$ is covered by at least one subset in $\mathcal{C}_i \subseteq \mathcal{C}$. Therefore, all elements of $\mathcal{X}$ are covered by $\mathcal{C}$.
\end{proof}

This result implies that solving each segmented subinstance independently and merging their solutions yields a valid solution to the original MSCP.

\subsection{Union--Find-Based Decomposition}
To efficiently identify the connected components of the co-occurrence graph, we employ a disjoint-set union (union--find) data structure \cite{GalilItaliano1991}.
Initially, each element $x \in \mathcal{X}$ is placed in its own set. For each subset $S \in \mathcal{F}$, all elements of $S$ are merged into the same set using union operations.
After processing all subsets, the resulting disjoint sets correspond exactly to the connected components of the co-occurrence graph. Fig.\ref{fig:cooccGraph} (right) illustrates the final result of the union--find process on a small MSCP instance. 
Initially, the nine elements of the universe were considered as nine independent groups (or trees), with no connections between them. After processing all the subsets, the elements that belong to the same connected component define independent segments of the universe.
Using union by rank and path compression—standard heuristics that respectively attach shallower trees to deeper ones and flatten the structure of the disjoint sets during find operations \cite{GalilItaliano1991}—the total time complexity of this preprocessing step is $O\!\left(\sum_{S \in \mathcal{F}} |S| \cdot \alpha(|\mathcal{X}|)\right)$, where $\alpha(\cdot)$ is the inverse Ackermann function which has near-constant complexity. The space complexity is linear in $|\mathcal{X}|$.
This preprocessing step is lightweight and can be applied prior to any optimization algorithm, making it suitable for large-scale MSCP instances.

\section{Integration with GRASP}
\label{sec:grasp}

The segmentation strategy described above is integrated into a Greedy Randomized Adaptive Search Procedure (GRASP) framework~\cite{Feo1995grasp,ResendeRibeiro2010}, and is further enhanced by incorporating the succinct set representation and bit-level set operations proposed by Delgado et al.~\cite{DelgadoFerradaNavarro2024}. This integration enables an efficient manipulation of candidate solution components during both the construction and local search phases. GRASP is an iterative metaheuristic that alternates between randomized greedy construction phases and solution improvement steps. The performance of the proposed approach is experimentally evaluated and compared against a classical Greedy algorithm, as illustrated in Figure~\ref{fig:grasp_rdp}. Algorithm~\ref{alg:grasp} outlines the overall sequential GRASP procedure, which relies on the randomized construction routine presented in Algorithm~\ref{alg:randSC} to generate partial and complete solutions.

Within this framework, the construction phase is carried out by Algorithm~\ref{alg:randSC}, which builds a solution iteratively by selecting candidate subsets according to a randomized greedy criterion. A key element of this interaction is the boolean flag \texttt{improve}, which regulates the balance between intensification and diversification during construction. This flag is queried by the procedure \textsc{FindBestCandidate} within Algorithm~\ref{alg:randSC}, but it is updated exclusively at the GRASP level in Algorithm~\ref{alg:grasp}, based on whether the solution obtained in the previous iteration improves upon the incumbent one.

When no improvement is observed in the previous GRASP iteration (\texttt{improve} = \texttt{False}), the construction phase promotes diversification by randomly selecting a subset with probability proportional to $(1 - f(S_i))$, where $f(S_i)$ denotes the evaluation score of the subset $S_i$, and $f$ is one of the functions in the equation (\ref{eq:evaluation-functions}). Conversely, when an improvement has been achieved (\texttt{improve} = \texttt{True}), the construction behaves greedily and deterministically selects the best-ranked candidate according to the evaluation function $f$, as in a standard greedy heuristic. This mechanism allows the algorithm to adapt its level of randomness dynamically, favoring exploration after stagnation and intensification when progress is being made.
The \texttt{improve} flag is updated at each GRASP iteration when comparing the newly constructed solution against the current best solution, while the construction routine itself remains stateless with respect to this decision.

\begin{algorithm}
\footnotesize
\caption{\textsc{RandSuccinctSC}: Randomized greedy construction procedure for the Minimum Set Cover Problem using a succinct set representation.}
\label{alg:randSC}
\begin{algorithmic}[1]
\Require Incomplete set cover $\mathcal{C}$, universe $\mathcal{X}$, \textsc{RowMap} structure, and \texttt{improve} flag
\Procedure{RandSuccinctSC}{$\mathcal{C}, \mathcal{X}, \textsc{RowMap}, \texttt{improve}$}
\State $\mathcal{U} \leftarrow \mathcal{X}$
\While{$\mathcal{U} \neq \varnothing$}
  \State $\textsc{Subsets} \leftarrow \textsc{FindSubsets}(\textsc{RowMap})$
  \State $f \leftarrow$ randomly selected evaluation function
  \State $S \leftarrow \textsc{FindBestCandidate}(\textsc{Subsets}, f, \mathcal{U}, \texttt{improve})$
  \State $\mathcal{C} \leftarrow \mathcal{C} \cup \{S\}$
  \State $\mathcal{U} \leftarrow \mathcal{U} \setminus S$
  \State $\textsc{RowMap} \leftarrow \textsc{updateRowMap}(S)$
\EndWhile
\State \Return $\mathcal{C}$
\EndProcedure
\end{algorithmic}
\end{algorithm}

Within this framework, universe segmentation is applied as a preprocessing step that decomposes the original MSCP instance into independent subinstances. The randomized greedy construction procedure described in Algorithm~\ref{alg:randSC} is then applied independently to each resulting subinstance, without modifying the greedy selection mechanism itself.

\begin{algorithm}
\footnotesize
\caption{\textsc{GRASP}: Sequential GRASP framework for the MSCP based on randomized greedy construction and local search.}
\label{alg:grasp}
\begin{algorithmic}[1]
\Require Universe $\mathcal{X}$ and family of subsets $\mathcal{F}$
\Procedure{GRASP}{$\mathcal{X}, \mathcal{F}$}
\State \texttt{num\_iter} $\gets$ number of GRASP iterations
\State \texttt{Max\_rm} $\gets$ maximum percentage of subsets to remove
\State $\textsc{RowMap} \leftarrow \textsc{createMap}(\mathcal{X}, \mathcal{F})$
\State $\textsc{Sort}(\textsc{RowMap})$
\State $\mathcal{C} \leftarrow \varnothing$
\State $\mathcal{U} \leftarrow \mathcal{X}$
\State $\mathcal{C} \leftarrow \textsc{RandSuccinctSC}(\mathcal{C}, \mathcal{U}, \textsc{RowMap}, \texttt{False})$
\State \texttt{improve} $\leftarrow$ \texttt{True}
\For{$i = 1$ \textbf{to} \texttt{num\_iter}}
  \State $\textsc{newC} \leftarrow \textsc{removeSets}(\mathcal{C}, \texttt{Max\_rm})$
  \State Update $\textsc{RowMap}$ and $\mathcal{U}$ w.r.t.\ $\textsc{newC}$
  \State $\textsc{newC} \leftarrow \textsc{RandSuccinctSC}(\textsc{newC}, \mathcal{U}, \textsc{RowMap}, \texttt{improve})$
  \State $\textsc{RemoveRedundantSets}(\textsc{newC})$
  \If{$|\textsc{newC}| < |\mathcal{C}|$}
    \State $\mathcal{C} \leftarrow \textsc{newC}$
    \State \texttt{improve} $\leftarrow$ \texttt{True}
  \Else
    \State \texttt{improve} $\leftarrow$ \texttt{False}
  \EndIf
\EndFor
\State \Return $\mathcal{C}$
\EndProcedure
\end{algorithmic}
\end{algorithm}

\subsection{Construction Phase and Evaluation Functions}
\label{sec:construction}

The construction phase follows a randomized greedy strategy in which candidate subsets are evaluated using a family of cost functions originally proposed by Lan and DePuy~\cite{LanDePuy2007}. 
At each iteration of the randomized construction procedure, an evaluation function $f$ is selected uniformly at random (Line~5 of Algorithm~\ref{alg:randSC}) and used to rank the candidate subsets.
The subset with minimum cost according to the selected function is then chosen, introducing diversification while preserving the greedy nature of the construction phase.

Let $\mathcal{U}$ denote the set of currently uncovered elements and let
$c_i = |S_i \cap \mathcal{U}|$ be the number of uncovered elements covered by subset $S_i$.
The following four evaluation functions are considered:
\begin{equation}
\label{eq:evaluation-functions}
f(S_i) \in \left\{ \frac{1}{c_i},\; \frac{1}{\sqrt{c_i}},\; \frac{1}{\log(c_i + 1)},\; \frac{1}{c_i^2} \right\}.
\end{equation}

Before executing the GRASP iterations, a preprocessing phase is applied to construct auxiliary data structures, including the element--subset incidence representation (\textsc{RowMap}) and the ordering of elements according to their degree.

\subsection{Computational Complexity of GRASP}
\label{sec:complexity}

The sequential GRASP algorithm consists of a preprocessing phase followed by an iterative optimization phase. The efficiency of both phases relies on the use of a succinct bit-level representation for set operations.

\paragraph{Preprocessing phase.}
The \textsc{RowMap} structure is constructed in $O(|\mathcal{X}|\cdot|\mathcal{F}|)$ time in the worst case. Sorting the elements of the universe by degree requires $O(|\mathcal{X}|\log|\mathcal{X}|)$ time.

\paragraph{Randomized construction phase.}
Let $W$ denote the machine word size. Using the succinct representation, set intersection, union, and difference operations can be performed in $O(|\mathcal{X}|/W)$ time. Each execution of Algorithm~\ref{alg:randSC} performs at most $\min\{|\mathcal{X}|,|\mathcal{F}|\}$ iterations. At each iteration, evaluating candidate subsets and updating coverage requires $O(|\mathcal{F}|\cdot|\mathcal{X}|/W)$ time, yielding a total construction complexity of
\[
O\!\left(\min\{|\mathcal{X}|,|\mathcal{F}|\}\cdot|\mathcal{F}|\cdot\frac{|\mathcal{X}|}{W}\right).
\]

\paragraph{GRASP iterations.}
Algorithm~\ref{alg:grasp} performs a fixed number of GRASP iterations, each dominated by a randomized construction phase. Therefore, the total running time of the sequential GRASP algorithm is dominated by repeated executions of Algorithm~\ref{alg:randSC}.

\subsection{Universe Segmentation and Independent Subinstances}

For segmentable MSCP instances, a preprocessing step can be applied to decompose the universe into independent components based on element co-occurrence. Each connected component defines a subinstance $(\mathcal{X}_i,\mathcal{F}_i)$ that can be solved independently using a standard GRASP procedure.

The partial solutions obtained for each subinstance are then merged to form a global solution. Since feasibility is preserved by construction, no additional repair or reconciliation step is required. This decomposition reduces the effective problem size and allows GRASP to focus on smaller, more structured subproblems.

\subsection{Parallel GRASP with Universe Segmentation}

In the preprocessing stage, two potentially parallelizable sections are identified. The first corresponds to the dominance of subsets (that is, redundant sets that are subsets of other subsets). The second section, called \texttt{createMap}, corresponds to the initialization of the \texttt{RowMap} structure. Furthermore, the idea behind segmentation is to form independent sub-universes that can be solved separately to obtain partial solutions to the original problem. With these considerations, a fully parallel algorithm, Algorithm 3, can be proposed.
While the standard GRASP algorithm operates on the full MSCP instance, universe segmentation enables another parallel extension of the framework. Algorithm~\ref{alg:grasp_su} presents the proposed \textsc{GRASP-SU} approach, which integrates segmentation of the universe with parallel GRASP execution.

In particular, Line~4 of Algorithm~\ref{alg:grasp_su} identifies groups of elements by applying a universe segmentation procedure that decomposes the original instance into independent components. In the main body of this work, this segmentation is performed using a union--find-based strategy that detects connected components induced by element co-occurrence. All experimental results reported in Figures~5 to~8 rely on this union--find-based segmentation.
Each resulting group is assigned to an independent GRASP execution, followed by a local search improvement phase. The partial covers obtained for all components are merged into a global solution, and a final redundancy removal step is applied to ensure minimality. Since feasibility is preserved by construction, no additional reconciliation is required.

For completeness, an alternative segmentation strategy based on forced partitioning using maximum spanning trees is analyzed separately in Appendix \ref{app:forced_partition}. This strategy replaces the union--find procedure at Line~4 of Algorithm~\ref{alg:grasp_su} and is shown to degrade performance.

\begin{algorithm}
\footnotesize
\caption{\textsc{GRASP-SU}: Parallel GRASP algorithm with universe segmentation for the MSCP.}
\label{alg:grasp_su}
\begin{algorithmic}[1]
\Require Universe $\mathcal{X}$, family of subsets $\mathcal{F}$, and number of threads $n_t$
\Procedure{GRASP-SU}{$\mathcal{X}, \mathcal{F}, n_t$}
\State $\textsc{RowMap} \leftarrow \textsc{createMap}(\mathcal{X}, \mathcal{F})$
\State $\textsc{Sort}(\textsc{RowMap})$
\State $\textsc{groups} \leftarrow \textsc{FindGroups}(\textsc{RowMap})$
\State $\textsc{SetCover} \leftarrow \varnothing$
\For{\textbf{each} \texttt{group} \textbf{in parallel}}
  \State $\mathcal{C} \leftarrow \varnothing$
  \State $\mathcal{U} \leftarrow \texttt{group.U}$
  \State $\textsc{map} \leftarrow \texttt{group.RowMap}$
  \State $\mathcal{C} \leftarrow \textsc{RandSuccinctSC}(\mathcal{C}, \mathcal{U}, \textsc{map}, \texttt{False})$
  \State $\textsc{LocalSearch}(\mathcal{C}, \mathcal{U}, \textsc{map})$
  \State $\textsc{SetCover} \leftarrow \textsc{SetCover} \cup \mathcal{C}$
\EndFor
\State $\textsc{RemoveRedundantSets}(\textsc{SetCover})$
\State \Return $\textsc{SetCover}$
\EndProcedure
\end{algorithmic}
\end{algorithm}

\subsubsection{Computational Complexity of GRASP-SU}
\label{sec:complexity_su}

Algorithm~\ref{alg:grasp_su} extends the sequential GRASP framework by exploiting universe segmentation and parallel execution. Let the segmentation produce $k$ connected components $(\mathcal{X}_1,\mathcal{F}_1),\ldots,(\mathcal{X}_k,\mathcal{F}_k)$.

The union--find-based segmentation step requires
\[
O\!\left(\sum_{S \in \mathcal{F}} |S| \cdot \alpha(|\mathcal{X}|)\right)
\]
time and $O(|\mathcal{X}|)$ space. Although asymptotically near-linear, experimental results show that this segmentation phase represents a significant fraction of the total runtime and limits overall scalability.

Each subinstance is then solved independently using Algorithm~\ref{alg:grasp}. Assuming a balanced segmentation, the parallel execution time after segmentation is dominated by the slowest subinstance:
\[
O\!\left(
\max_{1 \le i \le k}
\left\{
\min\{|\mathcal{X}_i|,|\mathcal{F}_i|\}
\cdot |\mathcal{F}_i| \cdot \frac{|\mathcal{X}_i|}{W}
\right\}
\right).
\]

Despite the segmentation overhead, substantial empirical speedups are observed for large and segmentable instances, as confirmed by the experimental results.

\section{Results and Discussion}
\label{sec:results}

This section reports the experimental evaluation of the proposed approaches. We first describe the computational environment and the implementations under comparison. All evaluated algorithms, including the Greedy baseline, GRASP, and the universe-segmentation-based variants, share the same succinct set representation and bit-level operations for manipulating subsets and uncovered elements, ensuring a fair and consistent experimental comparison.

\subsection{Setup}
\label{subsec:setup}

All experiments were conducted on a shared-memory multi-core server equipped with an AMD Ryzen\texttrademark\ Threadripper\texttrademark\ PRO~5975WX processor featuring 32 cores and 64 threads. The system is equipped with 128~GB of DDR4 RAM running at 3200~MHz, a 2~TB solid-state drive (SSD), and runs a 64-bit Arch Linux operating system with kernel version~6.2.10-arch1-1.
All algorithms were implemented in \texttt{C++} and compiled using \texttt{g++} version~9.3.0 with optimization level \texttt{-O3}. Parallel implementations rely exclusively on shared-memory parallelism via multi-threading, using OpenMP for thread management. To ensure fair comparisons, all methods employ the same succinct bit-level data structures for set representation.
Reported runtimes correspond to real (wall-clock) execution time, measured using the \texttt{std::chrono::high\_resolution\_clock}.

\subsection{Evaluated Algorithms}
\label{sec:algorithms}

We evaluate the following algorithms for the Minimum Set Cover Problem (MSCP): \textsc{Greedy}, \textsc{GRASP}, and \textsc{GRASP-SU}. All methods are implemented in \texttt{C++} using the same succinct bit-level data structures and preprocessing routines to ensure fair comparisons. The complete implementation, including source code, preprocessing scripts, and experimental setup, is publicly available in an open-source repository~\cite{mscp_repo}.

\begin{itemize}
\item \textsc{Greedy}: A deterministic greedy algorithm that iteratively selects the subset covering the largest number of currently uncovered elements. This method serves as a baseline for evaluating solution quality and runtime performance.

  \item \textsc{GRASP}: Implementation of Algorithm \ref{alg:grasp}, a sequential Greedy Randomized Adaptive Search Procedure that alternates between randomized greedy construction phases and local search improvement steps.
  
  \item \textsc{PAR-GRASP}: The parallel version of the \textsc{GRASP} implementation, parallelized as described in Section~\ref{sec:par_grasp}. 
  
\item \textsc{GRASP-SU}: A GRASP variant that integrates universe segmentation as a preprocessing step. Depending on the segmentation strategy employed, the resulting approach is referred to as \textsc{GRASP-UF} when a union--find-based method is used, or as \textsc{GRASP-MST} when a minimum spanning tree (MST)-based strategy is applied, as discussed in Appendix~\ref{app:forced_partition}. The segmentation decomposes the universe into independent subinstances, which can be solved concurrently using parallel GRASP executions. The resulting partial solutions are subsequently merged, and redundancies are removed to obtain a feasible global solution.

\end{itemize}

The forced MST-based partitioning strategy is intentionally excluded from this section and analyzed separately in Appendix~\ref{app:forced_partition} as an idea with potential that requires further research to fully assess its effectiveness.

Since the proposed approach includes configurable parameters such as \texttt{num\_iter} and \texttt{Max\_rm}, a preliminary parameter tuning phase was conducted using a trial-and-error methodology, which is commonly adopted in the context of metaheuristics where optimal parameter settings are often instance-dependent~\cite{LanDePuy2006}. 
This calibration phase was performed using the unicost instances described in Tables~\ref{tab:orlib_small} and~\ref{tab:orlib_large}, executing each instance 10 independent times and recording the best solution obtained.
We evaluated combinations of
\[
\texttt{num\_iter} \in \{200, 300, 500\}, \qquad
\texttt{Max\_rm} \in \{0.1, 0.3, 0.5, 0.7, 0.9\}.
\]

The results of this preliminary tuning phase are summarized in Table~1 (not reproduced here for brevity). 
Based on these experiments, the configuration \texttt{num\_iter} $=300$ and \texttt{Max\_rm} $=0.5$ consistently yielded the best trade-off between solution quality and computational effort across the tested instances. 
Therefore, these parameter values were fixed and used for all subsequent experiments reported in this section.

\subsection{GRASP versus Greedy on OR-Library Instances}
\label{sec:grasp_vs_greedy}

Before evaluating Algorithm~\ref{alg:grasp_su}, which summarizes the main procedure of the proposed approach, we first assess the effectiveness of the GRASP method integrated with the succinct representation and bit-level set operations introduced by Delgado et al.~\cite{DelgadoFerradaNavarro2024}, as detailed in Algorithm~\ref{alg:grasp}. Subsequently, the proposed GRASP approach is evaluated by comparing its performance against a classical Greedy algorithm on benchmark instances from the OR-Library~\cite{BeasleyORLibrary}. The Greedy algorithm provides a fast and deterministic baseline, selecting at each step the subset that covers the largest number of currently uncovered elements. In contrast, GRASP introduces randomized greedy constructions combined with local search, enabling the exploration of multiple solution trajectories.
The experiments consider unicost MSCP instances after preprocessing, which removes trivially covered elements and dominated subsets. The characteristics of the resulting instances are reported in Tables~\ref{tab:orlib_small} and~\ref{tab:orlib_large}.
The instances are grouped into two categories according to their size. Table~\ref{tab:orlib_small} reports small and medium instances from the \emph{scpe}, \emph{scpclr}, and \emph{scpcyc} families, while Table~\ref{tab:orlib_large} summarizes large-scale instances mainly drawn from the \emph{rail} and \emph{railway} families.
For each instance, the tables report the total number of elements $|\mathcal{X}|$ and subsets $|\mathcal{F}|$ in the original MSCP instance. After preprocessing, the number of elements covered by forced selections, denoted by $\mathcal{X}_{\text{cov}}$, and the number of remaining uncovered elements $\mathcal{X}_{\text{left}}$ are reported, satisfying
\[
|\mathcal{X}| = \mathcal{X}_{\text{cov}} + \mathcal{X}_{\text{uncov}}.
\]
Similarly, the tables report the number of subsets forced into the solution during preprocessing $\mathcal{F}_{\text{inc}}$, the number of excluded (discarded) subsets $\mathcal{F}_{\text{ex}}$, and the number of remaining subsets $\mathcal{F}_{\text{left}}$, such that
\[
|\mathcal{F}| = \mathcal{F}_{\text{inc}} + \mathcal{F}_{\text{ex}} + \mathcal{F}_{\text{left}}.
\]
Therefore, all Greedy and GRASP experiments are executed on the reduced instances defined by $(\mathcal{X}_{\text{uncov}}, \mathcal{F}_{\text{left}})$. 
The last column of each table reports the \emph{Best Known Solution} (BKS) available at~\cite{Wang2021}. 

\begin{table*}[ht!]
\centering
\footnotesize
\caption{OR-Library unicost instances after preprocessing. Columns report the number of elements and subsets before and after preprocessing, as well as the best known solution (BKS).}
\label{tab:orlib_small}
\begin{tabular}{lrrrrrrrr}
\hline
Instance & $|\mathcal{X}|$ & $\mathcal{X}_{\text{cov}}$ & $\mathcal{X}_{\text{uncov}}$ & $|\mathcal{F}|$ & $\mathcal{F}_{\text{inc}}$ & $\mathcal{F}_{\text{exc}}$ & $\mathcal{F}_{\text{left}}$ & BKS \\
\hline
scpe1     & 50    & 0 & 50    & 500  & 0 & 8  & 492  & 5  \\
scpe2     & 50    & 0 & 50    & 500  & 0 & 8  & 492  & 5  \\
scpe3     & 50    & 0 & 50    & 500  & 0 & 6  & 494  & 5  \\
scpe4     & 50    & 0 & 50    & 500  & 0 & 12 & 488  & 5  \\
scpe5     & 50    & 0 & 50    & 500  & 0 & 10 & 490  & 5  \\
scpclr10  & 511   & 0 & 511   & 210  & 0 & 0  & 210  & 25 \\
scpclr11  & 1023  & 0 & 1023  & 330  & 0 & 0  & 330  & 23 \\
scpclr12  & 2047  & 0 & 2047  & 495  & 0 & 0  & 495  & 23 \\
scpclr13  & 4095  & 0 & 4095  & 715  & 0 & 0  & 715  & 23 \\
scpcyc06  & 240   & 0 & 240   & 192  & 0 & 0  & 192  & 60 \\
scpcyc07  & 672   & 0 & 672   & 448  & 0 & 0  & 448  & 144 \\
scpcyc08  & 1792  & 0 & 1792  & 1024 & 0 & 0  & 1024 & 344 \\
scpcyc09  & 4608  & 0 & 4608  & 2304 & 0 & 0  & 2304 & 772 \\
scpcyc10  & 11520 & 0 & 11520 & 5120 & 0 & 0  & 5120 & 1792 \\
scpcyc11  & 28160 & 0 & 28160 & 11264 & 0 & 0 & 11264 & 3968 \\
\hline
\end{tabular}
\end{table*}

\begin{table*}[t]
\centering
\footnotesize
\caption{OR-Library railway instances after preprocessing. The table reports the reduction achieved in both universe size and number of subsets, together with the best known solution (BKS).}
\label{tab:orlib_large}
\begin{tabular}{lrrrrrrrr}
\hline
Instance & $|\mathcal{X}|$ & $\mathcal{X}_{\text{cov}}$ & $\mathcal{X}_{\text{uncov}}$ & $|\mathcal{F}|$ & $\mathcal{F}_{\text{inc}}$ & $\mathcal{F}_{\text{exc}}$ & $\mathcal{F}_{\text{left}}$ & BKS \\
\hline
rail507  & 507  & 18  & 489  & 63009   & 7  & 37630 & 25372 & 96  \\
rail516  & 516  & 62  & 454  & 47311   & 27 & 8657  & 38627 & 134 \\
rail582  & 582  & 17  & 565  & 55515   & 7  & 28105 & 27403 & 125 \\
rail2536 & 2536 & 32  & 2504 & 1081841 & 5  & 270162& 811674& 377 \\
rail2586 & 2586 & 129 & 2457 & 920683  & 42 & 487757& 432884& 518 \\
rail4284 & 4284 & 76  & 4208 & 1092610 & 17 & 426255& 666338& 591 \\
rail4872 & 4872 & 270 & 4602 & 968672  & 70 & 488292& 480310& 877 \\
\hline
\end{tabular}
\end{table*}

\subsubsection*{Relative Deviation Analysis}

To show the quality of the solutions, most experimental graphs show the Relative Percentage Deviation (RPD) obtained by Greedy and GRASP across all OR-Library instances, computed as $(|C|-\text{BKS})/ \text{BKS}$, where $|C|$ corresponds to the cardinality of the solution.\\

The Figure~\ref{fig:grasp_rdp} presents a comparison between the sequential \texttt{GRASP} algorithm and the \texttt{Greedy} algorithm in terms of solution cardinality and execution time for unicost instances. \texttt{GRASP} consistently achieves low RPD values, mostly below 5\%, reaching the best known solution in 11 out of the 15 evaluated instances. However, it can be observed that \texttt{GRASP} tends to perform worse on very large instances. In contrast, the \texttt{Greedy} algorithm exhibits notably higher RPD values, reaching up to nearly 40\% in certain instances, and is equal to or worse than \texttt{GRASP} in all cases.
The right-hand graph shows the execution times for both strategies. \texttt{Greedy} proves to be considerably faster, being up to two orders of magnitude more efficient in execution time than \texttt{GRASP}. Nevertheless, despite the higher computational cost the improvement in solution quality achieved by \texttt{GRASP} justifies this increase in runtime. Finally, both algorithms exhibit a clear increasing trend in execution time as instance size grows.
\begin{figure}
\centering
\includegraphics[width=0.95\textwidth]{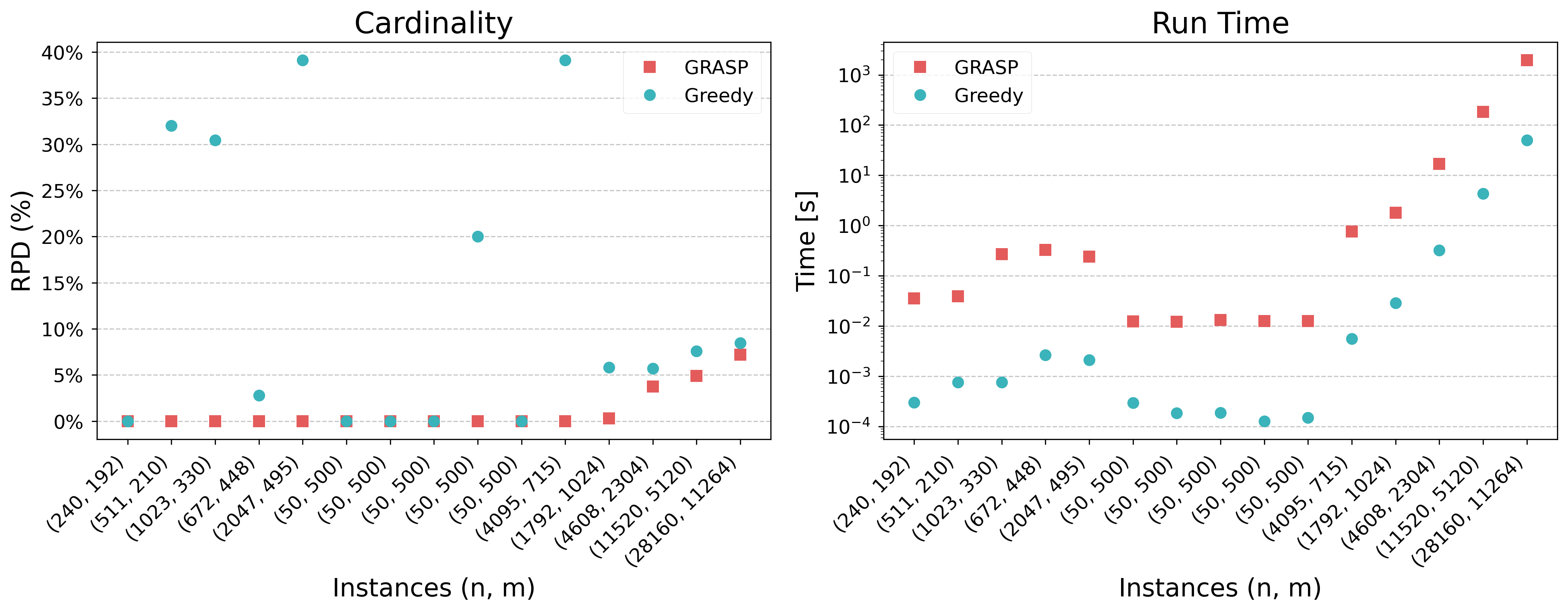}
\caption{Relative percentage deviation (RPD) obtained by Greedy and GRASP on OR-Library MSCP instances.}
\label{fig:grasp_rdp}
\end{figure}

These results demonstrate that GRASP consistently outperforms in quality the Greedy heuristic, particularly on large and structurally complex instances. Additionally, the reduction cardinality of GRASP often approaching or matching the BKS.

\subsubsection{Parallel Performance Analysis of GRASP}
\label{sec:par_grasp}

In the GRASP algorithm (Algorithm~\ref{alg:grasp}), two main sections are amenable to parallelization and form the basis of the \texttt{PAR-GRASP} variant.
The first corresponds to the \texttt{createMap} procedure (line~4), which builds the \texttt{RowMap}  structure.
The second involves sorting this structure according to element degree (line~5).
In addition, the dominance check among subsets is also parallelized; however, this operation is performed during a preprocessing phase executed prior to the GRASP iterations, and it represents a significant improvement in the parallelization of the entire process.

Figure~\ref{fig:grasp_rdp_rail} presents the performance of the algorithms on large-scale instances. For this evaluation, \texttt{PAR-GRASP} was executed using 32 threads CPU cores. In terms of cardinality (left graph), a pattern similar to that observed for USCP instances is evident. The \texttt{PAR-GRASP} algorithm maintains an RPD below 20\% in most instances. However, as instance size increases, a slight upward trend in RPD is observed, indicating a performance degradation of \texttt{PAR-GRASP} in higher-dimensional scenarios. Meanwhile, the \texttt{Greedy} algorithm again shows inferior performance, with RPD exceeding 20\% in nearly all instances and consistently more than 7\% worse than \texttt{PAR-GRASP} across all cases. Regarding execution time (right graph), the advantage of the \texttt{Greedy} approach is once again apparent, achieving results up to one order of magnitude faster than \texttt{PAR-GRASP}.

\begin{figure}
\centering
\includegraphics[width=0.95\textwidth]{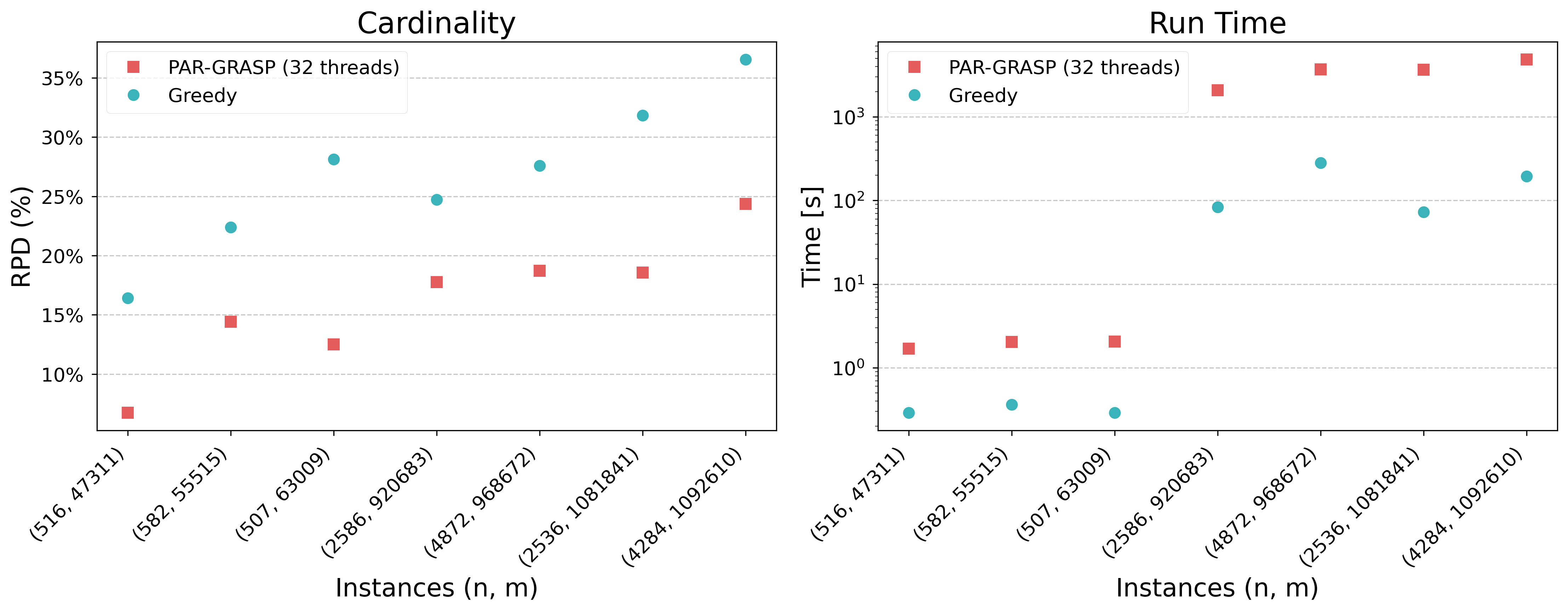}
\caption{Relative percentage deviation for railway instances, comparing Greedy and GRASP.}
\label{fig:grasp_rdp_rail}
\end{figure}

\subsection{Results for GRASP-UF}
\label{sec:grasp_uf}

The goal of universe segmentation is to identify independent sub-universes that can be solved separately, yielding partial solutions that can be combined into a feasible solution for the original problem. In this section, we analyze the performance of the universe-segmentation-based approach \textsc{GRASP-UF}, as described in Algorithm~\ref{alg:grasp_su}.

\begin{figure}
\centering
\includegraphics[width=0.95\textwidth]{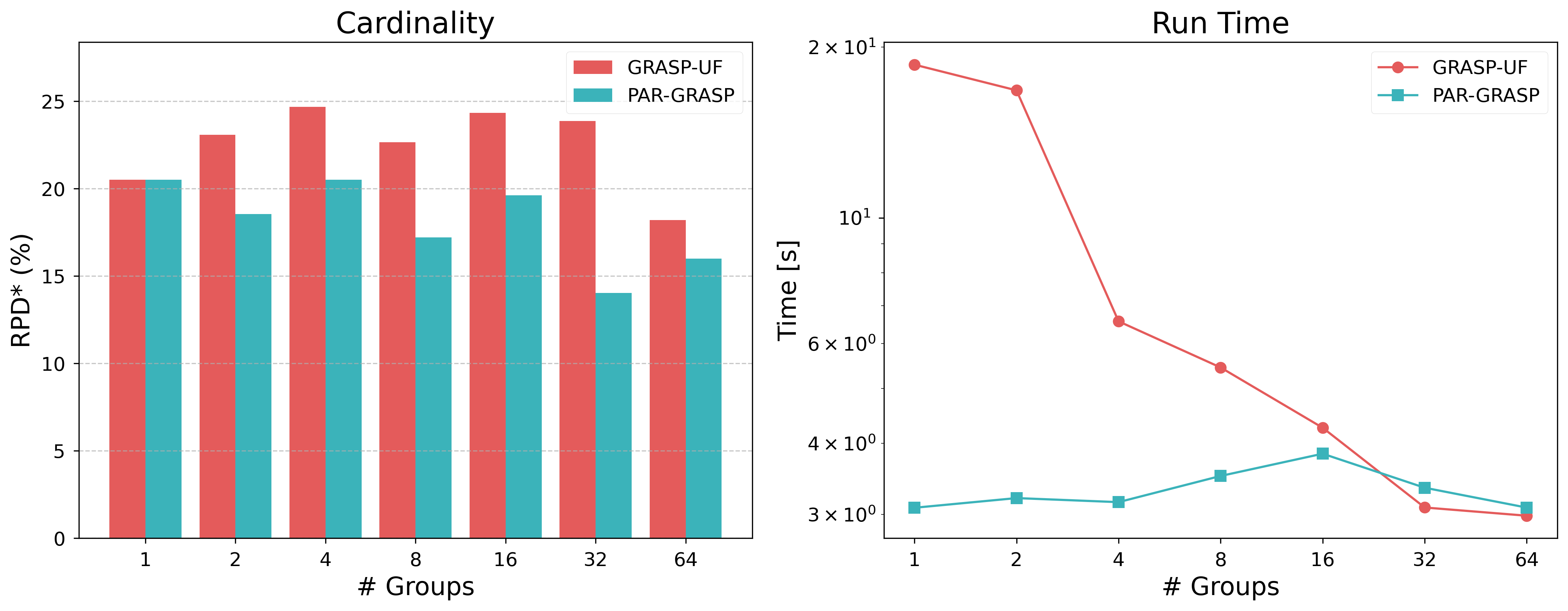}
\caption{Solution cardinality and execution time of \textsc{GRASP-UF} and \textsc{PAR-GRASP}, both executed with 32 CPU cores, on randomly generated MSCP instances of size $(n,m) = (10{,}000, 20{,}000)$.
The number of universe groups produced by segmentation varies in $\{1,2,4,8,16,32,64\}$.
The metric RPD$^\ast$ denotes the relative percentage deviation with respect to the solution cardinality obtained by the Greedy algorithm.}
\label{fig:uf_groups}
\end{figure}

For experiments based on synthetic instances, where no \emph{Best Known Solution} (BKS) is available, solution quality is assessed using the Relative Percentage Deviation (RPD) with respect to the Greedy algorithm, which serves as a deterministic baseline.
Given a solution of cardinality $|\mathcal{C}|$ produced by \textsc{GRASP-UF}, and the corresponding Greedy solution of cardinality $|\mathcal{C}_{\text{Greedy}}|$, the RPD is redefined as
$ \mathrm{RPD}^\ast =
\frac{|\mathcal{C}_{\text{Greedy}}| - |\mathcal{C}|}
{|\mathcal{C}_{\text{Greedy}}|}. $
Therefore, the higher the RPD, the better the quality of the algorithm.
This metric is used consistently in the analysis of the experimental results reported in Figures~\ref{fig:uf_groups}, ~\ref{fig:uf_4} and ~\ref{fig:uf_32}, allowing a direct comparison of solution quality across different segmentation configurations and levels of parallelism.
Figure ~\ref{fig:uf_groups} shows the number of connected components detected by the union--find preprocessing phase for different instances, and based on this number, the algorithms \textsc{PAR-GRASP} and \textsc{GRASP-UF} are compared.

Therefore, to further evaluate scalability, we generated large synthetic instances of MSCP to validate the behavior of parallel GRASP algorithms with and without independent group (or partition) searches in the universe.
Figures~\ref{fig:uf_groups} to~\ref{fig:uf_32} present results obtained on randomly generated instances.
Figure ~\ref{fig:uf_speedup} shows the time of the two main parallelization stages, indicating that universe segmentation with union--find (line 4 of Algorithm ~\ref{alg:grasp_su}) represents the largest fraction of the total execution time for all instances.
Despite this cost, segmentation significantly reduces the effective problem size, leading to overall improvements in scalability and solution quality.

\begin{figure}
\centering
\includegraphics[width=0.95\textwidth]{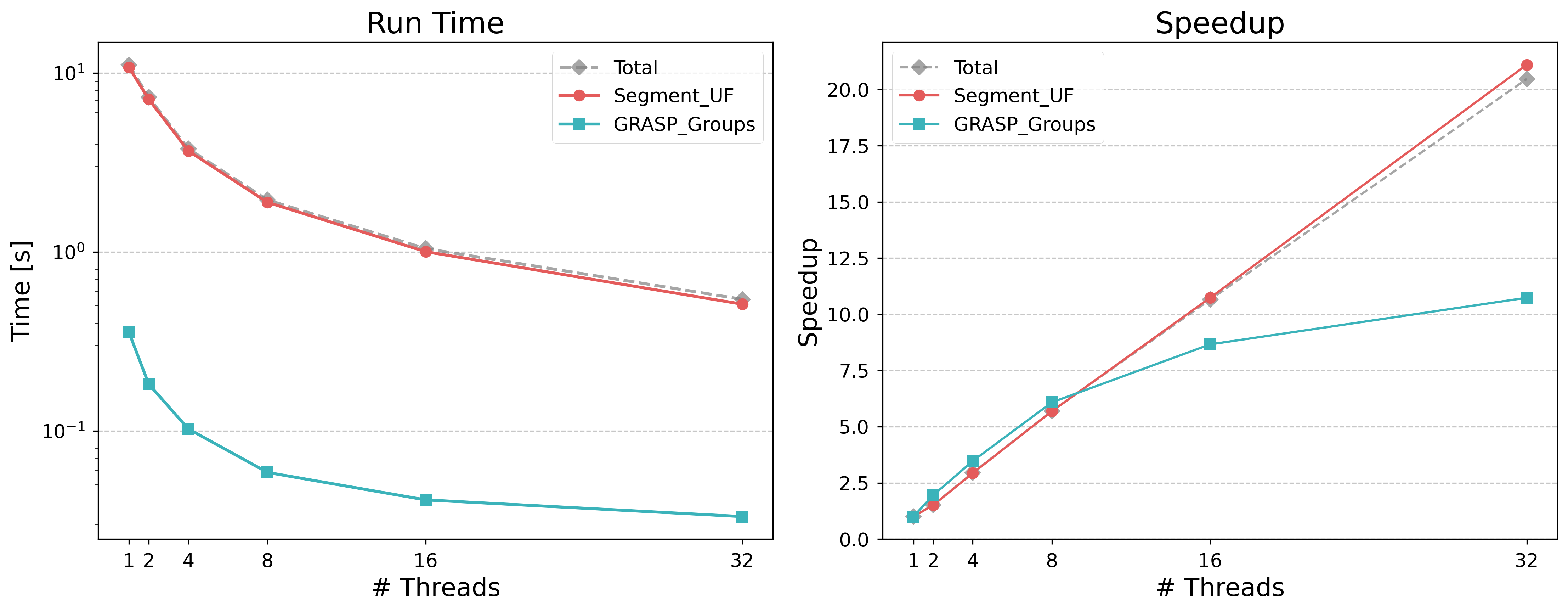}
\caption{Speedup obtained by GRASP-UF for different algorithmic phases for an instance with $(n,m) = (10{,}000, 20{,}000)$ and 32 groups using up to 32 threads.}
\label{fig:uf_speedup}
\end{figure}

\begin{figure}
\centering
\includegraphics[width=0.95\textwidth]{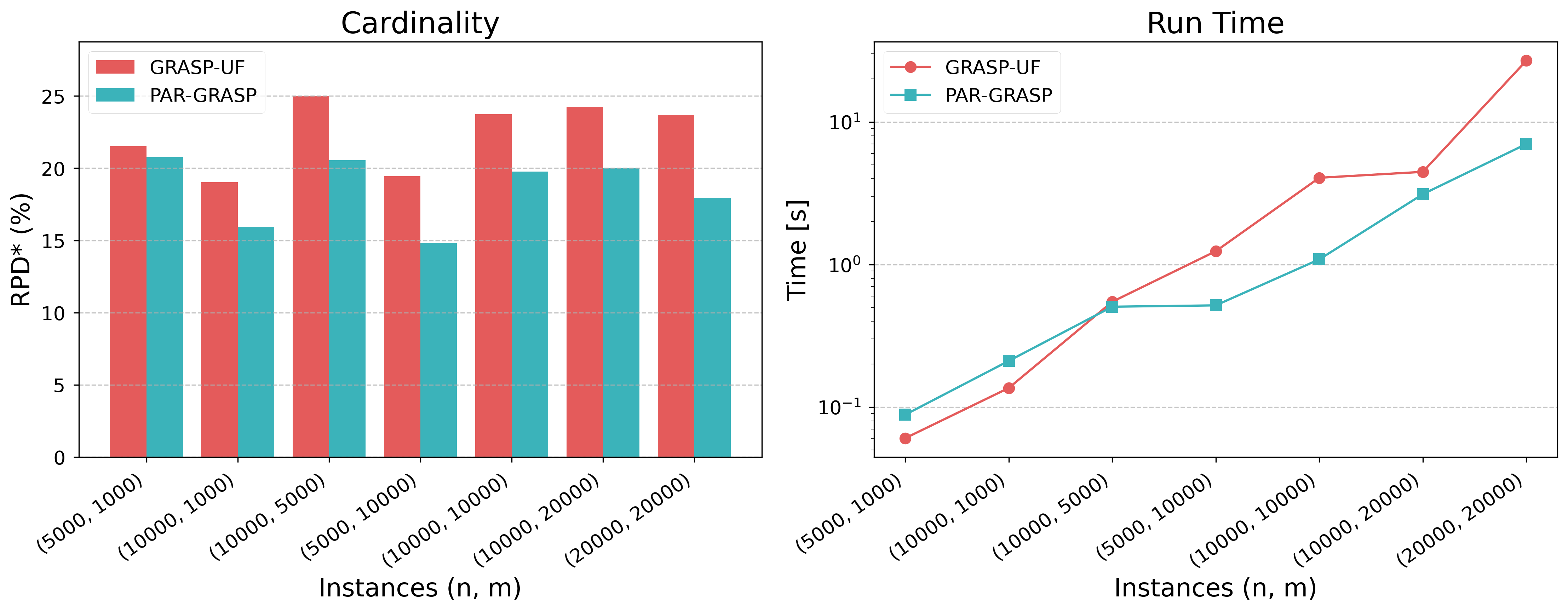}
\caption{Performance comparison of PAR-GRASP and GRASP-UF using 4 threads on synthetic MSCP instances.}
\label{fig:uf_4}
\end{figure}

\begin{figure}
\centering
\includegraphics[width=0.95\textwidth]{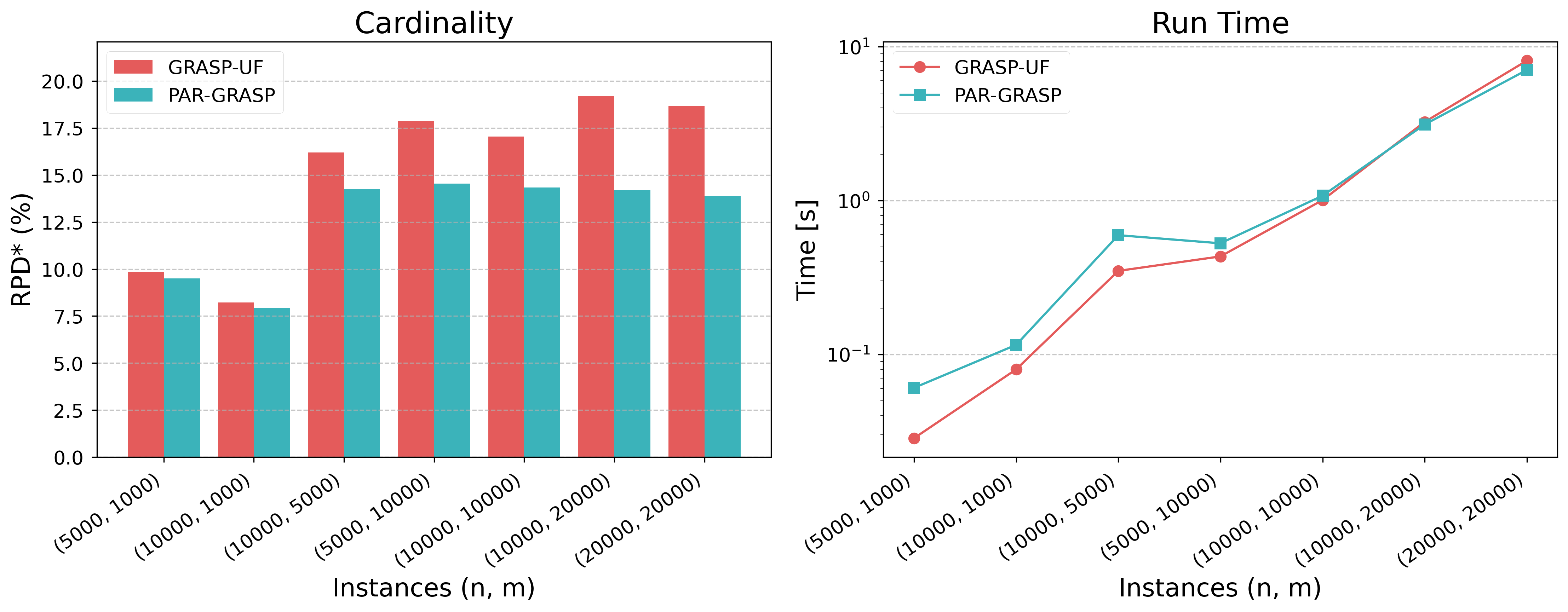}
\caption{Performance comparison of PAR-GRASP and GRASP-UF using 32 threads on synthetic MSCP instances.}
\label{fig:uf_32}
\end{figure}

\section{Conclusions and Future Work}
\label{sec:conclusions}

This work studied the exploitation of intrinsic structural properties of the Minimum Set Cover Problem (MSCP) to improve the scalability of GRASP-based heuristics and quality of the solution. 
In particular, we introduced the notion of \emph{universe segmentability}, which captures the existence of independent components induced by element co-occurrence, and proposed a preprocessing strategy based on union--find to identify such components with near-linear theoretical complexity.

By integrating universe segmentation into both sequential and parallel GRASP frameworks, we showed that large MSCP instances can often be decomposed into smaller, independent subinstances that can be solved more efficiently without compromising feasibility. 
Experimental results on standard OR-Library instances confirm that this approach leads to significant reductions in effective problem size, competitive solution quality, and substantial parallel speedups on shared-memory multi-core architectures. At the same time, the experimental analysis reveals that the segmentation phase itself may represent a dominant fraction of the total execution time for large instances, thus becoming a limiting factor for scalability.

Apart from the positive results, this work found cases where the proposed techniques underperformed,  providing important methodological insights. In particular, the current forced universe partitioning strategies based on maximum spanning trees and balanced bi-partitions were found to be ineffective. Although such approaches produce partitions of comparable size, they fail to preserve the independence required for safe MSCP decomposition, as subsets frequently span across partitions. As discussed in the appendix, these results suggest that enforcing balance can be detrimental, and that respecting structural coherence among elements is more critical than achieving evenly sized subinstances.

These observations motivate several directions for future research. First, alternative MST-based or graph-theoretic segmentation strategies that relax balance constraints should be explored, potentially allowing highly unbalanced or recursive decompositions when strong structural dependencies exist. Second, adaptive segmentation schemes that dynamically refine or merge components across GRASP iterations may help reduce preprocessing overhead while preserving decomposition benefits. Finally, extending the proposed framework to weighted MSCP variants, as well as investigating implementations on GPU or large-scale HPC platforms 
could potentially further improve scalability and robustness.

\section*{Acknowledgment}
This work was supported by ANID FONDECYT grants $\#11221029$, $\#1221357$ and by the Patag\'on supercomputer of Universidad Austral de Chile (FONDEQUIP EQM180042). 

\bibliographystyle{cas-model2-names}
\bibliography{sample}

\clearpage
\appendix
\section{Insights and Future Directions on Forced Universe Partitioning}
\label{app:forced_partition}

This appendix reports negative experimental results obtained from a forced universe partitioning strategy for the Minimum Set Cover Problem (MSCP).
The approach is based on graph-theoretic principles and aims to decompose the universe into balanced sub-instances using a maximum spanning tree.
Although conceptually appealing, this strategy was found to be ineffective in practice.
The following analysis clarifies the reasons for this behavior and motivates future research directions.

\subsection{Balanced Bipartition via Maximum Spanning Trees}

Let $(\mathcal{X}, \mathcal{F})$ be an MSCP instance and let $G = (\mathcal{X}, E)$ denote the co-occurrence graph induced by $\mathcal{F}$, where vertices correspond to elements and an edge $(x_i,x_j)$ exists if both elements appear together in at least one subset.
Each edge is assigned a weight reflecting the strength of interaction between elements, defined as
\[
w_{ij} = \left| \{ S \in \mathcal{F} \mid \{x_i, x_j\} \subseteq S \} \right|.
\]

A maximum spanning tree (MST) $T$ of $G$ is computed to preserve the strongest co-occurrence relationships.
To force a decomposition of the universe, one edge of $T$ is removed.
Among all candidate edges, the selected cut is the one that yields the most balanced bipartition, that is, two components of similar size or total weight whenever possible.
Figure~\ref{fig:mst_partition} illustrates this process on a representative MSCP instance.

\begin{figure}
\centering
\includegraphics[width=0.85\textwidth]{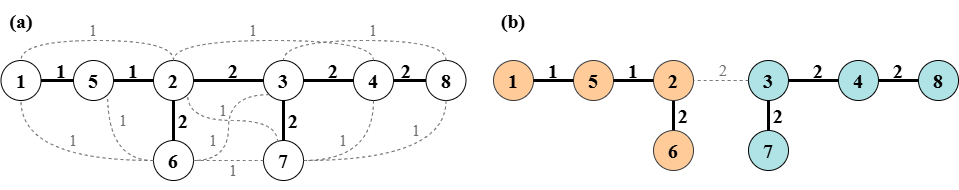}
\caption{Forced balanced bipartition of the \emph{preprocessed} MSCP instance derived from Figure~\ref{fig:example}.
The co-occurrence graph corresponds to the reduced universe obtained after preprocessing, in which four elements have been eliminated, subset $S_6$ has been fixed in the solution, and subset $S_7$ has been discarded.
Edge weights represent the number of subsets in which pairs of remaining elements co-occur.
A maximum spanning tree is constructed and its edges are evaluated as candidate cuts to induce a balanced bipartition.
The partition obtained by cutting edge $(u,v,w)=(2,3,2)$ yields two groups with total weights $W_1=5$ and $W_2=6$.
Despite the apparent balance, the resulting components remain structurally coupled through shared subsets, illustrating why forced balanced partitioning fails to produce independent MSCP subinstances even after preprocessing.}
\label{fig:mst_partition}
\end{figure}

\subsection{Algorithmic Description}

Algorithm~\ref{alg:forced_partition} summarizes the forced universe bipartition strategy based on a maximum spanning tree.

\begin{algorithm}
\footnotesize
\caption{Forced balanced universe bipartition using a maximum spanning tree.}
\label{alg:forced_partition}
\begin{algorithmic}[1]
\Require Universe $\mathcal{X}$, family of subsets $\mathcal{F}$
\State Construct weighted co-occurrence graph $G = (\mathcal{X}, E, w)$
\State Compute a maximum spanning tree $T$ of $G$
\State Select the edge $e^\ast \in T$ whose removal yields the most balanced split
\State Remove $e^\ast$ from $T$, obtaining components $\mathcal{X}_1$ and $\mathcal{X}_2$
\State Define subinstances $(\mathcal{X}_1,\mathcal{F}_1)$ and $(\mathcal{X}_2,\mathcal{F}_2)$
\State \Return $\{(\mathcal{X}_1,\mathcal{F}_1), (\mathcal{X}_2,\mathcal{F}_2)\}$
\end{algorithmic}
\end{algorithm}

\subsection{Empirical Performance of the MST}

The practical consequences of this strategy are illustrated in Figure~\ref{fig:mst_rpd}, which reports the relative degradation in solution quality resulting from the application of a balanced MST-based partitioning.

\begin{figure}
\centering
\includegraphics[width=0.9\textwidth]{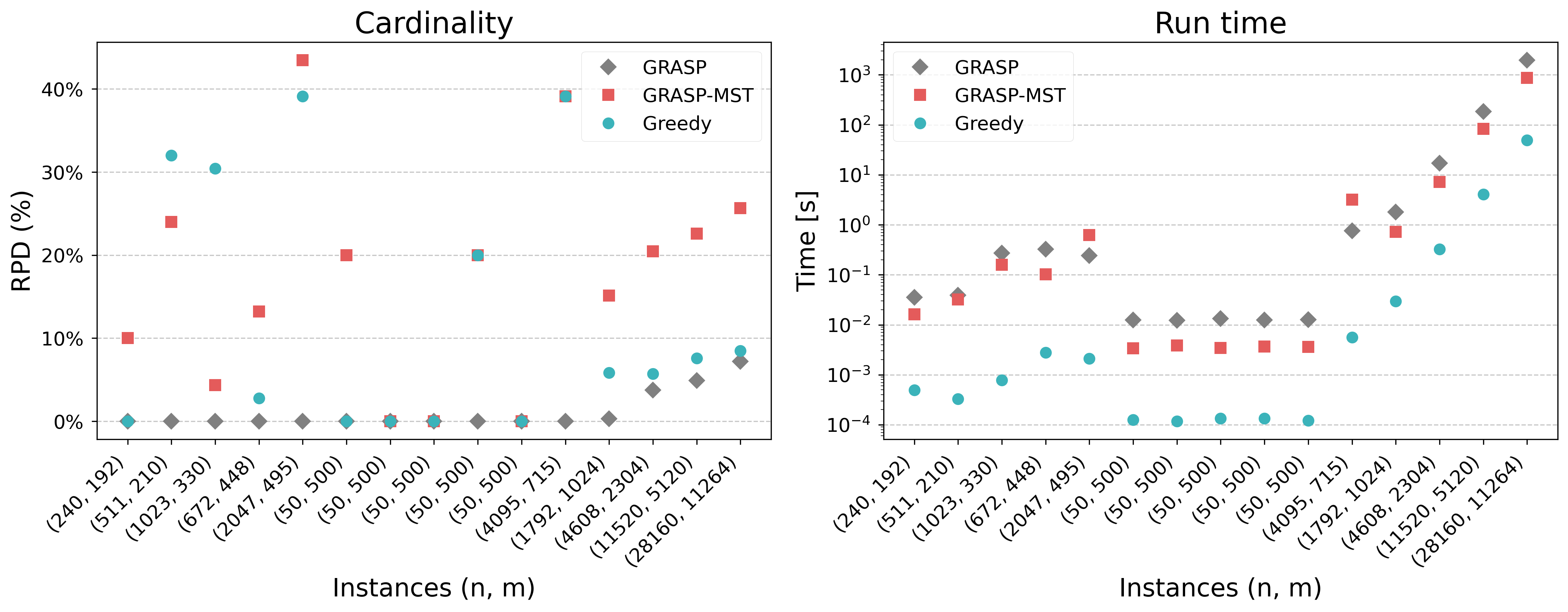}
\caption{Impact of forced balanced universe partitioning based on a maximum spanning tree.
The figure reports the relative degradation in solution quality observed when a balanced MST-based bipartition is applied prior to solving the resulting subinstances independently.
Although the partition produces components of comparable size, the induced structural coupling between elements leads to inferior solutions.}
\label{fig:mst_rpd}
\end{figure}

\subsection{Implications and Future Directions}

These results indicate that the observed performance degradation is not inherently due to the use of maximum spanning trees, but rather to the constraint of enforcing balanced partitions.
In the MSCP, preserving structural coherence among elements appears to be more important than distributing elements evenly across subinstances.
Promising directions for future research include relaxing balance constraints in MST-based partitioning, allowing highly unbalanced cuts when strong dependencies exist.
Additionally, recursive MST-based decomposition strategies may be explored, in which large components are progressively subdivided while structural coupling remains weak.
Such adaptive and structure-aware partitioning strategies may better align with the intrinsic structure of MSCP instances.

\end{document}